\documentclass[twoside]{article}
\usepackage[accepted]{aistats2016}

\usepackage{amsmath, amsthm, amscd, amssymb}
\usepackage{mathtools}
\usepackage{dsfont}
\usepackage{color}
\usepackage{enumitem}
\usepackage{graphicx}
\usepackage{caption,subcaption}
\usepackage[round]{natbib}
\usepackage{appendix}
\usepackage[ruled]{algorithm2e}
\usepackage{fancyhdr}

\renewcommand{\cite}{\citep}

\newcommand{\cov}{\mathrm{Cov}}
\newcommand{\n}{\mathbf{n}}
\newcommand{\T}{T}   

\newcommand{\y}{\mathbf{y}}
\newcommand{\Z}{\mathbf{Z}}

\newcommand{\m}{\mathbf{m}}
\renewcommand{\b}{\mathbf{b}}
\renewcommand{\r}{\mathbf{r}}
\newcommand{\A}{{A}}
\renewcommand{\P}{P}
\newcommand{\pib}{\boldsymbol\pi}
\newcommand{\mub}{\boldsymbol\mu}
\newcommand{\thetab}{\boldsymbol{\theta}}

\newcommand{\E}[1]{\ensuremath\mathbb{E}\left[#1\right]}
\renewcommand{\Pr}[1]{\ensuremath\mathrm{Pr}\left(#1\right)}
\newcommand{\indep}{\bot\!\!\!\!\bot}
\newtheorem{proposition}{Proposition}

\newtheorem{theorem}{Theorem}
\newcommand{\var}[1]{\ensuremath\mathrm{Var}\left(#1\right)}

\newcommand{\CLS}{{\ensuremath{\text{CLS}}}}
\newcommand{\MoM}{{\ensuremath{\text{MoM}}}}

\newcommand{\eat}[1]{}

\newcommand{\R}{\mathbb{R}}
\DeclareMathOperator{\diag}{diag}
\DeclareMathOperator{\argmin}{argmin}

\usepackage{tikz}
\usetikzlibrary{shapes,arrows,positioning,fit}
\tikzstyle{rectangle}=[draw=black,thick]
\tikzstyle{circ}=[circle,draw=black,thick,minimum size=3ex, inner sep=.2ex]
\tikzset{>=latex}

\begin{document}

%

%

\twocolumn[

\aistatstitle{Consistently Estimating Markov Chains with Noisy Aggregate Data}

\aistatsauthor{ Garrett Bernstein \And Daniel Sheldon }

\aistatsaddress{ University of Massachusetts Amherst } ]

\begin{abstract}

We address the problem of estimating the parameters of a time-homogeneous Markov chain given only noisy, aggregate data. This arises when a population of individuals behave independently according to a Markov chain, but individual sample paths cannot be observed due to limitations of the observation process or the need to protect privacy. Instead, only population-level counts of the number of individuals in each state at each time step are available. When these counts are exact, a conditional least squares (CLS) estimator is known to be consistent and asymptotically normal. We initiate the study of method of moments estimators for this problem to handle the more realistic case when observations are additionally corrupted by noise. We show that CLS can be interpreted as a simple ``plug-in'' method of moments estimator. However, when observations are noisy, it is not consistent because it fails to account for additional variance introduced by the noise. We develop a new, simpler method of moments estimator that bypasses this problem and is consistent under noisy observations.
\end{abstract}

\section{Introduction} 
\label{sec:introduction}

The problem of learning from aggregate data has arisen over the years in diverse fields including machine learning, 
statistics and econometrics,
and social sciences.
In each case, the goal is to make inferences or fit models at the level of individuals when the only available data is at the population level, for example, counts of the number of individuals with certain properties. Example applications include: learning models of bird migration from citizen science count data~\cite{Sheldon:2013aa,Liu:2014aa}, learning models of human mobility from data that is aggregated to maintain privacy~\cite{sun2015message}, fitting models of voter turnout and demography from census data~\cite{King:2013aa,Flaxman:2015}, and modeling of credit risk from historical data about the proportions of institutions with different credit ratings~\cite{Jones:2005aa}. 

We consider the particular problem of estimating the parameters of a time-homogeneous Markov chain from noisy, aggregate data. In this problem, $N$ individuals behave independently according to the same Markov chain for $T$ time steps, and, at each time step, a noisy observation is made of the number of individuals in each state. The entire process is repeated $K$ times, and the goal is to recover the transition probabilities of the Markov chain. We assume that the chains are started from the stationary distribution, so the entire process is stationary and the only parameters to estimate are the transition probabilities. The fundamental questions we seek to address are: Is it possible to recover transition probabilities given only aggregate data, and, if so, under what conditions? How much aggregate data is necessary to obtain accurate parameter estimates, and how does this compare to estimation with individual-level data?

This problem has previously arisen in two distinct settings. First, it has a long history in statistics and econometrics, where it is sometimes referred to as estimating Markov chains from ``macro'' data. The traditional approach is a conditional least squares (CLS) estimator~\cite{miller1952finite,Madansky:1959aa,Lee:1970aa,Aigner:1974aa,Van-Der-Plas:1983aa,Kalbfleisch:1983aa}. When observations are exact, it is known that the CLS estimator is consistent and asymptotically normal as $T \rightarrow \infty$~\cite{Van-Der-Plas:1983aa}, which answers one of our questions: it is indeed possible to learn transition probabilities with only aggregate data.
However, little is known about CLS or other estimators for the case when observations are noisy. It was previously presumed that CLS is not consistent when observations follow a simple binomial noise model~\cite{MacRae:1977aa}. Our analysis highlights exactly why this is true and suggests alternate estimators that are consistent.

The second setting where our problem has appeared is in the context of \emph{collective graphical models} (CGMs)~\cite{Sheldon:2011aa}, a recent formalism for inference and learning with aggregate data. In CGMs, individual-level data are generated by any graphical model, and observations are made of contingency tables (counts of the number of times different variable configurations appear in the population). The model we consider here is the special case of CGMs where the individual model is a time-homogeneous Markov chain. Unlike the prior work on aggregate Markov chains, CGMs explicitly model noise in the observations~\cite{Sheldon:2011aa}. However, the CGM literature has focused primarily on inference, and uses expectation maximization for learning, which is effective in certain cases but provides no guarantees~\cite{Sheldon:2011aa,Sheldon:2013aa,Liu:2014aa,sun2015message}.
Our work contributes the first learning method with guarantees of any kind for a subclass of CGMs.

In this paper, we initiate the study of method of moments estimators for aggregate Markov chains to explicitly deal with imperfect observations. This is an important practical issue: aggregate data are rarely complete surveys of a population, so they should be considered \emph{noisy} counts of the number of individuals in each state. The method of moments viewpoint yields a number of useful observations. First, we show that CLS can be interpreted as a simple ``plug-in'' method of moments estimator. However, when observations are noisy, it is not consistent because it fails to account for additional variance introduced by the noise. Second, we develop a new and simpler method of moments estimator that bypasses the issue faced by CLS and is consistent with noisy observations.

Our primary contribution is to develop the first estimator with comprehensive theoretical guarantees for estimation of Markov chains from noisy, aggregate data. We show that our new method of moments estimator is consistent in both the \emph{time-average} ($T \rightarrow \infty$, fixed $K$) and \emph{ensemble-average} ($K \rightarrow \infty$, fixed $T$) settings with observations from a broad class of noise models. We show through both theoretical and empirical results that the squared error of our estimator decays as $\mathcal{O}(1/TK)$. One previous work~\cite{MacRae:1977aa} considered the problem of estimating Markov chains with aggregate data corrupted by binomial noise. Based on the presumption that CLS was not consistent, which we confirm here, MacRae proposed a ``limited information'' estimator. However, that estimator relies on signal in the time-varying marginals of the process, and we show that it is \emph{not} consistent for the stationary process we consider here.

In the remainder of the paper, we describe the model and problem statement (Section~\ref{sec:problem_statement_and_notation}), introduce method of moments estimators, including CLS, for perfect observations (Section~\ref{sec:mom}), develop our new estimator for noisy observations (Section~\ref{sec:noisy}), prove theoretical guarantees (Section~\ref{sec:theory}), and then evaluate various estimators empirically (Section~\ref{sec:experiments}).

\section{Model and Problem Statement} 
\label{sec:problem_statement_and_notation}

Our model for individuals is an ergodic, time-homogeneous Markov chain on state space $\{1, \ldots, S\}$. The probability for the state trajectory $x_1, \ldots, x_T$ is:
\[
p(x_1, \ldots, x_T) = \pi(x_1) \prod_{t=1}^{T-1} P(x_t, x_{t+1}),
\]
where $\P$ is the $S \times S$ transition matrix, whose entries $P(i,j)$ specify the probability of transitioning from state $i$ to state $j$, and $\pib$ is the stationary distribution, i.e., the unique vector $\pib$ such that $\pib^T \P = \pib^T$.

Aggregate data is generated by $N$ individuals independently transitioning from state to state according to the same Markov chain. Let $x_t^{(m)}$ be the state of the $m$th individual at time $t$, and let $\n_t$ be the vector with entries $n_t(i) = \sum_{m=1}^N [x_t^{(m)}=i]$ that count the number of individuals in each state at time $t$ (here $[\, \cdot\, ]$ denotes an indicator function). The vectors $\n_1, \ldots, \n_T$ constitute the aggregate data. We further assume that observations are noisy, so the observed data are vectors $\y_1, \ldots, \y_T$ that depend probabilistically on the aggregate data through a noise model $p(\y_t \mid \n_t)$. For example, a model we consider later is one where individuals are observed with probability $\alpha$, so $y_t(i) \sim \text{Binomial}(n_t(i), \alpha)$. The entire aggregate process is repeated $K$ times, independently, to yield data vectors $\y_t^{(k)}$ for $k \in \{1,\ldots,K\}$ and $t \in \{1,\ldots,T\}$. The goal is to estimate $\P$ from the noisy, aggregate data $\{\y_t^{(k)}\}$. The plate diagram for one realization of the aggregate process is shown in Figure~\ref{plate_model}.

\textbf{Stationarity.} 
We assume that chains start in the stationary distribution for simplicity and because this is the most difficult setting for estimation. A main focus of our work is the case when $T \rightarrow \infty$, in which case our estimators and asymptotic guarantees apply without modification for arbitrary initial distributions; we simply need to wait slightly longer for the chains to mix. It is also easy to modify our estimators to explicitly model the non-stationary initial distributions; we describe this in Section~\ref{sec:non-stationary}. To see why estimation is most difficult in the stationary setting, and to set up our later analysis, note that because each chain is started in the stationary distribution and is time-homogeneous, the process is (strongly) stationary: the joint distribution of $(x_{t_1}, \ldots, x_{t_k})$ is the same as that of $(x_{t_1 + a}, \ldots, x_{t_k + a})$ for any subset of times $t_1, \ldots, t_k$ and any time lag $a$. In particular, the marginal distribution of $x_t$ is equal to $\pib$ for all $t$. This means that, marginally, each vector $\n_t$ is a multinomial draw from the stationary distribution; it is quite clear that we can estimate $\pib$ accurately from these vectors, but much less obvious that they contain enough information to estimate $\P$.

\begin{figure}
  \center{
    \begin{tikzpicture}[node distance = 2cm, auto]		
      
      \node (x1) [circ]  {$x_1^{(m)}$};
      \node (x2) [circ, right of=x1] {$x_2^{(m)}$};
      \node (dots) [right of=x2] {$\dots$};
      \node (xt) [circ, right of=dots] {$x_T^{(m)}$};
      \node (plate) [draw,inner sep=2ex, thick, fit=(x1)(x2)(dots)(xt)]{};
      \node (plate_label) [yshift=-1ex,xshift=-4ex] at (plate.north east) {\tiny$m=1:N$};		
      \draw [->] (x1) -- (x2);
      \draw [->] (x2) -- (dots);
      \draw [->] (dots) -- (xt);

      \node (n1) [circ, below=.75cm of x1] {$\n_1$};
      \node (n2) [circ, below=.75cm of x2] {$\n_2$};
      \node (nt) [circ, below=.75cm of xt] {$\n_T$};

      \draw [->] (x1) -- (n1);
      \draw [->] (x2) -- (n2);
      \draw [->] (xt) -- (nt);

      \node (y1) [circ, fill=lightgray, below =.5cm of n1] {$\y_1$};
      \node (y2) [circ, fill=lightgray, below =.5cm of n2] {$\y_2$};
      \node (yt) [circ, fill=lightgray, below =.5cm of nt] {$\y_T$};

      \draw [->] (n1) -- (y1);
      \draw [->] (n2) -- (y2);
      \draw [->] (nt) -- (yt);		

    \end{tikzpicture}}
  \caption{Plate Model}
  \label{plate_model}
\end{figure}
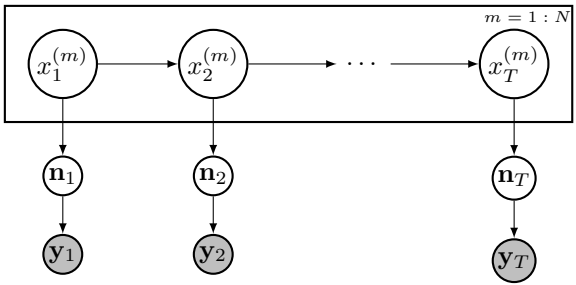

\section{Method of Moments}
\label{sec:mom}

We will now describe how the method of moments applies to this
problem by showing that the first and second moments of the aggregate process, which can be estimated from data, uniquely identify the marginals of the process, which in turn uniquely identify the transition matrix.

\paragraph{Marginals.}

Let $\mub_t \in \R^S$ and $\mub_{t,t+1} \in \R^{S \times S}$ be the vector and matrix,
respectively, of \emph{single} and \emph{pairwise} marginals of the
Markov process, defined by:
\begin{align*}
\mub_t(i) &= \Pr{x_t = i} \\
\mub_{t,t+1}(i,j) &= \Pr{x_t = i, x_{t+1}= j}.
\end{align*}
Since our process is stationary, the marginals $\mub_t$ and $\mub_{t,t+1}$ do not depend on
$t$ and $\mub_t = \pib$ for all $t$. However, we retain the dependence on $t$
to make it more clear how the results generalize to non-stationary
Markov chains. 

It is clear that one can recover the transition matrix $\P$ from the pairwise marginals, since $P(i,j) = \Pr{x_{t+1} = j | x_{t} = i} = \mu_{t,t+1}(i,j)/\mu_t(i)$, or, in matrix form:
\[
\P  = \diag(\mub_t)^{-1} \mub_{t,t+1}.
\]
Thus, if we can consistently estimate the marginals, we can recover
$\P$.

\paragraph{Moments.}
We will see that the marginals are encoded in a simple way in the first and second moments of the \emph{aggregate} Markov process $\n_1, \ldots, \n_T$. Denote the mean vector at each time step (the first moments) by $\m_t = \E{\n_t}$, and denote the covariance matrices at time lags zero and one, respectively, by $\Sigma_t := \var{\n_t}$ and $\Sigma_{t,t+1} := \cov(\n_t, \n_{t+1})$. Denote the non-central second moments at time lags zero and one, respectively, by $\Lambda_t := \E{\n_t \n_t^T}$ and $\Lambda_{t,t+1} := \E{\n_t \n_{t+1}^T}$. Each of these quantities is estimable from the aggregate data. The following proposition shows that they also encode the single and pairwise marginals of the process in a straightforward way:
\begin{proposition}
\label{prop:moments}
For all $t$, the following are true:
\begin{align}
\label{one} \m_t &= N \mub_t \\
\label{two} \Sigma_t &= N\big(\diag(\mub_t) - \mub_t \mub_t^T\big) \\
\label{three} \Sigma_{t,t+1} &= N\big(\mub_{t,t+1} - \mub_t \mub_{t+1}^T\big) \\
\label{four} \Lambda_t &= N\big(\diag(\mub_t) + (N-1)\mub_t \mub_t^T\big) \\
\label{five} \Lambda_{t,t+1} &= N\big(\mub_{t,t+1} + (N-1)\mub_t \mub_{t+1}^T\big).
\end{align}
\end{proposition}
\begin{proof}
The marginal distribution of $\n_t$ is $\text{Multinomial}(N, \mub_t)$. Equations~\eqref{one} and \eqref{two} are well known formulas for the mean vector and covariance matrix of this multinomial distribution. Equation~\eqref{three} was proved by \citet{Liu:2014aa}. Equations~\eqref{four} and \eqref{five} follow from the previous equations and the fact that $\E{\n_s \n_t^T} = \cov(\n_s, \n_t^T) + \E{\n_s}\E{\n_t^T}$.
\end{proof}

\subsection{Conditional Least Squares}

CLS \cite{miller1952finite} is based on the observation that, for all $t$, we expect $\n_{t}$ to be close to its conditional expectation given $\n_{t-1}$, which is $\n_{t-1}^T \P$. In other words, we expect the following for all $t$:
\[
\n_{t-1}^T \P \approx \n_{t}.
\]
We collect this into a least squares system by letting $X = \begin{bmatrix}\n_1\; \n_2\; \ldots\; \n_{T-1}\end{bmatrix}^T$ and $Y = \begin{bmatrix}\n_2\; \n_3\; \ldots\; \n_{T}\end{bmatrix}^T$. We can then write the least squares problem as:
\[
\hat{\P}_\CLS = \argmin_{\P} \| X \P - Y \|_F^2.
\]
where $\|\cdot\|_F^2$ is the squared Frobenius norm, which is the sum of squares of the matrix entries.
It is well known that the solution is given by 
\[
\hat{\P}_\CLS = (X^TX)^{-1} X^TY.
\]

\paragraph{Interpretation as Method of Moments} 

We first note that $X^TX$ and $X^TY$ are proportional to the straightforward empirical estimators of the non-central moments $\Lambda_t$ and $\Lambda_{t,t+1}$:
\begin{align*}
(X^T X)_{ij} &= \sum_{t=1}^{T-1} n_t(i) n_t(j) \\
(X^T Y)_{ij} &= \sum_{t=1}^{T-1} n_t(i) n_{t+1}(j).
\end{align*}
To emphasize this point, we write $\hat{\Lambda}_t = (T-1)^{-1} X^T X$ and $\hat{\Lambda}_{t,t+1} = (T-1)^{-1} X^T Y$, and can then equivalently write the CLS estimator as
\[
\hat{\P}_\CLS = \hat{\Lambda}_t^{-1} \hat{\Lambda}_{t,t+1}.
\]
The following proposition shows that this is, in fact, a simple ``plug-in'' estimator for $P$ using the empirical estimates of the non-central moments in place of their true counterparts. 

\begin{proposition}$P = \Lambda_t^{-1} \Lambda_{t,t+1}$.
\label{prop:P}
\end{proposition}
\begin{proof}
It is straightforward to show that $\Lambda_t$ is invertible (see
supplementary material). We then verify the equivalent statement that $\Lambda_t P = \Lambda_{t,t+1}$:
	\begin{align*}
		\Lambda_t P &= N\big(\diag(\mub_t) + (N-1)\mub_t \mub_t^T\big)P \\
					&= N\big(\diag(\mub_t)P + (N-1)\mub_t \mub_t^TP\big) \\
					&= N\big(\mub_{t,t+1} + (N-1)\mub_t \mub_{t+1}^T\big).
	\end{align*}

	This matches the definition of $\Lambda_{t,t+1}$ (Equation \ref{five}), as desired.
\end{proof}



\subsection{A Simpler Estimator}
Proposition~\ref{prop:moments} (Equation~\ref{three}) suggests a simpler way to recover the pairwise marginal $\mub_{t,t+1}$, and hence $\P$, from the moments of the aggregate process:
\[
\mub_{t,t+1} =  N^{-1}\Sigma_{t,t+1} + \mub_t \mub_{t+1}^T.
\]
So, a simple method of moments estimator is obtained by replacing the moments by their empirical estimators:
\begin{equation}
\label{eq:mom}
\hat{P}_{\MoM} = \diag(\hat{\mub}_{t})^{-1}\big(N^{-1}\hat{\Sigma}_{t,t+1} + \hat{\mub}_t \hat{\mub}_{t+1}^T\big).
\end{equation}
We will describe the estimator in more detail (see Algorithm~\ref{alg:mom}) after introducing noisy observations in the next section. We will observe that this simpler estimator performs better asymptotically than CLS for noisy observations.

\section{Noisy Observations}
\label{sec:noisy}

So far we have assumed that the exact aggregate data vectors $\n_1, \n_2, \ldots, \n_T$ are observed and thus we can estimate the moments. What happens if we only observe noisy vectors $\y_1, \ldots, \y_T$? It is not hard to imagine that for simple enough noise models, it will still be possible to consistently estimate the moments of $\n$; in particular, this is always possible if there is a bijection between the moments of $\n$ and $\y$. Our approach will be to define a large class of noise models for which it is straightforward to recover certain moments of $\n$ from those of $\y$. In the process of setting up this connection we will see exactly why CLS is \emph{not} consistent when applied to the noisy vectors $\y_1, \ldots, \y_T$.

\eat{
For example, consider the model in which each individual in the population is counted independently with probability $\alpha$, so that $y_t(i) \mid n_t(i) \sim \text{Binomial}(n_t(i), \alpha)$. One can easily derive a bijection between the moments of $\n$ and $\y$; in particular, $\E{\n_t} = \alpha^{-1}\E{\y_t}$ and $\cov(\n_t, \n_{t+1}) = \alpha^{-2}\cov(\y_t, \y_{t+1})$. 
\[
\var{\n_s} = 
\]
}

\textbf{Example: Inconsistency of CLS.}
We will start with a very simple example to illustrate the main points. Suppose that $\y_t = \n_t + \epsilon_t$ where $\epsilon_t$ is a zero-mean noise vector that is independent of all other random variables. Then, one can easily verify the following relationships:
\begin{align*}
\E{\y_t} &= \E{\n_t} \\
\var{\y_t} &= \var{\n_t}  + \var{\epsilon_t} \\
\E{\y_t \y_t^T} &= \E{\n_t \n_t^T} + \var{\epsilon_t} \\
\cov(\y_t, \y_{t+1}) &= \cov(\n_t, \n_{t+1}) \\
\E{\y_t \y_{t+1}^T} &= \E{\n_t \n_{t+1}^T}.
\end{align*}
In particular, if we just pretend the data is not noisy and use $\y$ in our estimators in place of $\n$, this \emph{almost} does the right thing. It is only the zero-lag moments $\var{\y_t}$ and $\E{\y_t\y_t^T}$ that are incorrect. But note that CLS uses such a moment in its estimator, so if we blindly run CLS on $\y$ in place of $\n$, the asymptotic result will be:
\[
P_\CLS = \big(\Lambda_{t} + \var{\epsilon_t}\big)^{-1} \Lambda_{t,t+1} \neq P.
\] 
Thus, CLS is not consistent in this noise model. In contrast, the simpler model we propose here uses only the mean vector and the time-lagged second moments, so it is consistent without modification when $\n$ is replaced by $\y$.

\textbf{General Noise Models.}
The following proposition delineates a much broader class of noise models for which we can recover the needed moments of $\n$ from those of $\y$.
\begin{proposition}
\label{prop:noise}
Suppose the noise model $p(\y \mid \n)$ satisfies the following two conditions:
\vspace{-8pt}
\begin{enumerate}[itemsep=0pt,label=(\roman*)]
\item $\y_t \indep \y_s \mid \n_t$ for $s \neq t$ (independent noise),
\item $\E{\y_t \mid \n_t} \!= \! \A_t \n_t$, for known, invertible matrix $\A_t$.
\end{enumerate}
\vspace{-8pt}
Then the following moments of $\n$ can be recovered from those of $\y$:
\vspace{-8pt}
\begin{enumerate}[itemsep=0pt,label=(\roman*)]
\item $\E{\n_t} = A_t^{-1} \E{\y_t}$ for all $t$,
\item $\cov(\n_s, \n_t) = A_s^{-1} \cov(\y_s, \y_t) A_t^{-T}$ for $s \neq t$,
\item $\E{\n_s \n_t^T}  = A_s^{-1} \E{\y_s \y_t^T} A_t^{-T}$ for $s \neq t$.
\end{enumerate}
\vspace{-5pt}
\end{proposition}
The proof can be found in the supplementary material. Note that the formulas for recovering the moments use very little information about the noise model---only the linear form of the conditional mean $\E{\y_t \mid \n_t}$---so very detailed knowledge of the noise mechanism is not necessary. Also note that Proposition~\ref{prop:noise} does \emph{not} give formulas for $\var{\n_t}$  and $\E{\n_t \n_t^T}$; as in our simple preceding example, these are more complicated, and require greater knowledge of the noise model than the conditional mean. Proposition~\ref{prop:noise} can be further generalized to the case where there is an affine relationship $\E{\y_t \mid \n_t} = A_t\n_t + \b_t$, but the expressions become more complicated.

This proposition suggests that the matrix $A_t$ must be known in advance. For many noise models, however, it is possible to consistently estimate this matrix from the data: our experiments demonstrate this for both binomial and additive noise models.

\subsection{Examples}
\label{sub:noisy_observations_examples}

We now give several examples of noise models that meet the conditions of Proposition~\ref{prop:noise} and lead to simple ways to recover the moments of $\n$ from those of $\y$.

\textbf{Binomial or Poisson noise.}
A simple example is the one in which each individual is observed independently with probability $\alpha$. A small variant is the case when the number of individuals counted in each state is a Poisson random variable with mean proportional to the true number. These noise models are given by:
\begin{align*}
y_t(i) \mid n_t(i) &\sim \text{Binomial}\big(n_t(i), \alpha\big) \\
y_t(i) \mid n_t(i) &\sim \text{Poisson}\big(\alpha \cdot n_t(i)\big).
\end{align*}
In both cases, we have $\E{\y_t \mid \n_t} = \alpha \n_t$, which satisfies the conditions of Proposition~\ref{prop:noise} with $A_t = \alpha I$ for all $t$. Thus, we can recover the moments of $\n$ as:
\begin{align*}
\E{\n_t} &= \alpha^{-1}\E{\y_t} \\
\E{\n_t \n_{t+1}^\T} &= \alpha^{-2}\E{\y_t \y_{t+1}^\T} \\
\cov(\n_t, \n_{t+1}) &= \alpha^{-2}\cov(\y_t, \y_{t+1}).
\end{align*}

\textbf{Additive Noise.}
We have already discussed additive noise models of the form $\y_t = \n_t + \epsilon_t$, for example:
\[
\epsilon_t(i) \sim \text{Normal}(0, \sigma^2), \quad
\epsilon_t(i) \sim \text{Laplace}(b).
\]
These are special cases of Proposition~\ref{prop:noise} with $A_t = I$, so we can just substitute $\y$ for $\n$ in our estimator. 
Gaussian noise is a very common model for measurement error. Both Gaussian and Laplace noise are used in mechanisms that explicitly add noise to count data such as ours prior to release to guarantee differential privacy~\cite{dwork2013algorithmic,sun2015message}. Our results show that it is possible to consistently learn with private data under these noise mechanisms.

\textbf{State-Dependent Detection Probability.}
Another interesting model occurs when the probability that an individual is counted varies by state. Suppose that the detection probability in state $i$ is $\alpha_i$, so that
\[
y_t(i) \mid n_t(i) \sim \text{Binomial}\big(n_t(i), \alpha_i\big).
\]
Then we have $\E{\y_t \mid \n_t} = A \n_T$ with $A = \diag([\alpha_1, \ldots, \alpha_S]^\T)$, so we can also apply method of moments in this case (as long as the detection probabilities are known or can be estimated).


\subsection{Putting It Together: Estimation with Noisy Data}
\begin{algorithm}[tb]
\SetKwInOut{Input}{Input}
\SetKwInOut{Output}{Output}
\Input{noise model matrix $A$, population size $N$, data vectors $\y_t^{(k)}$ for $t=1,\ldots, T$ and $k = 1, \ldots, K.$}
\Output{estimated transition matrix $\hat{P}$}
\begin{enumerate}[leftmargin=*,itemsep=0pt]
\item Estimate mean of noisy data:
\vspace{-8pt}
$$
\hat{\m}_\y := \frac{1}{TK} \sum_{t=1}^T \sum_{k=1}^K \y_t^{(k)}.
$$
\vspace{-14pt}
\item Estimate mean of true counts: $\hat{\m}_\n := A^{-1}\hat{\m}_\y$.
\item Normalize: $\hat{\mub} := \hat{\m}_\n/\mathbf{1}^\T \hat{\m}_\n$.
\item Estimate time-lagged covariance of true counts:
\vspace{-4pt}
\[
\hat{\Sigma} := A^{-1}\Bigg( \frac{1}{(T-1)K} \sum_{t=1}^{T-1}
\sum_{k=1}^K \r_t^{(k)} (\r_{t+1}^{(k)})^\T \Bigg)A^{-\T}
\]
\vspace{-10pt}

where $\r_t^{(k)} := \y_t^{(k)} - \hat{\m}_\y$.
\item Estimate transition matrix:
\vspace{-4pt}
\[
\hat{P} := \diag(\hat{\mub})^{-1}\big(N^{-1} \hat{\Sigma} + \hat{\mub}\hat{\mub}^T\big).
\]
\end{enumerate}
\vspace{-8pt}
\caption{Method of Moments with Noise}
\label{alg:mom}
\end{algorithm}
\setlength{\textfloatsep}{10pt}

The detailed procedure for method of moments with noisy data is given in Algorithm 1. In this algorithm, we assume that the entire process is stationary, including the noise model, so there is a single matrix $A$ such that $\E{\y_t \mid \n_t} = A \n_t$ for all $t$. The algorithm accepts $N$ and $A$ as inputs. With exact observations, these parameters are known: $A=I$ and $N$ is the total count at any time step. With noisy observations, $A$ and $N$ are not known, but can often be estimated easily. For example, in the case of binomial noise, let $z_t = \sum_{i=1}^S y_t(i)$ be the total number of individuals observed at time $t$. Then the $z_t$ variables are iid $\text{Binomial}(N, \alpha)$ random variables, from which $N$ and $\alpha$ can be consistently estimated by various methods~\cite{Blumenthal1981}.

\subsection{Non-Stationary Processes}
\label{sec:non-stationary}
Our ideas can also be extended to non-stationary Markov chains, e.g., when the individual chains are started from arbitrary distributions, or when the transition probabilities are time-varying. To see this, note that Equation~\eqref{eq:mom} shows that we can recover the (potentially time-varying) transition matrix $\P_t$ from the moments $\mub_t$ and $\Sigma_{t,t+1}$ at time $t$ for \emph{any} $t$. If these moments are time-varying, instead of averaging over all $t$ as we do in Algorithm~\ref{alg:mom}, we can construct separate estimates $\hat{P}_t$ at each time step and then combine those to obtain a final parameter estimate. This makes the most sense when parameters defining the transition probabilities are shared across time steps. One such example is the case when $P$ is time-homogenous but the chains are started in non-stationary distributions. In this case, we can estimate $P$ by averaging the time-specific estimates: $\hat{P} = \frac{1}{T}\sum_{t=1}^T \hat{P}_t$. Another example is when transition probabilities are time-varying but compactly parameterized as $P_t(\thetab)$ for a finite-dimensional $\thetab$.  In this case, we can solve for the best parameters in a least squares sense: $\hat{\thetab} = \argmin_{\thetab} \sum_{t=1}^T \|P_t(\thetab) - \hat{P}_t\|^2$. 

\begin{figure*}[t]
    \centering
      \begin{subfigure}{.49\textwidth}
      \centering
        \includegraphics[width=0.88\textwidth]{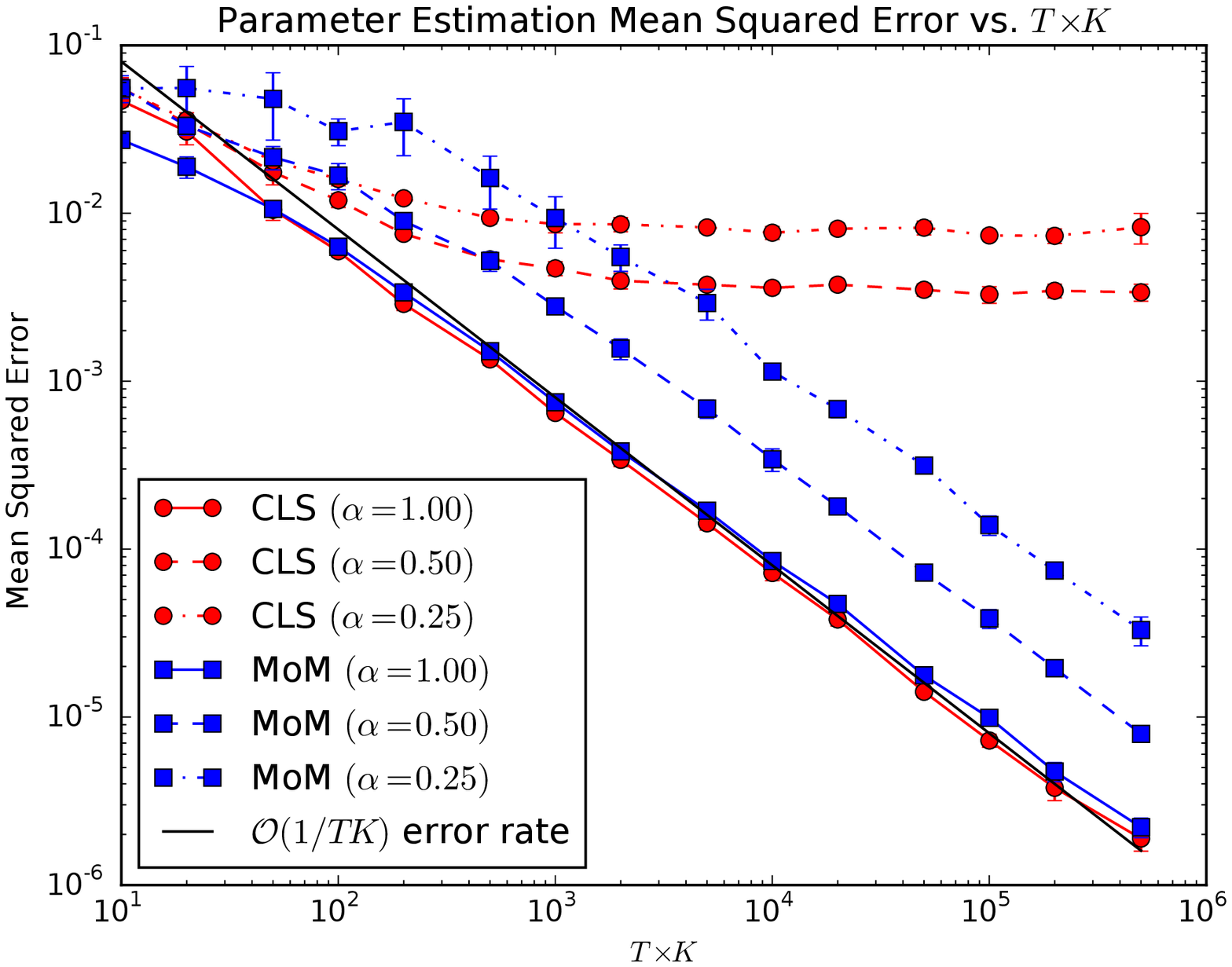}
        \caption{}
        \label{fig:TK_error_binomial}
      \end{subfigure}
      \hfill
      \begin{subfigure}{.49\textwidth}
      \centering
        \includegraphics[width=0.88\textwidth]{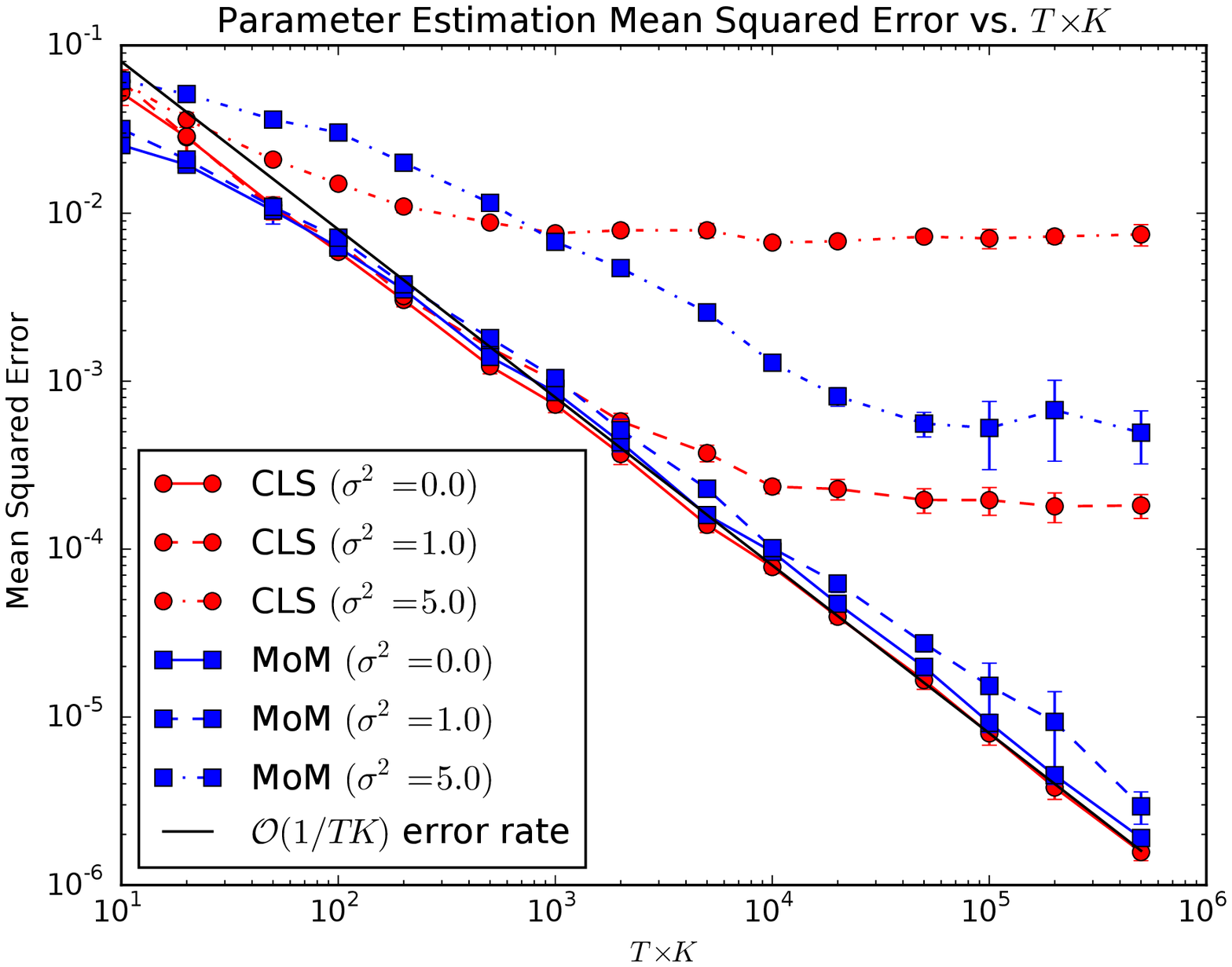}
        \caption{}
        \label{fig:TK_error_additive}
      \end{subfigure}

      \caption{Parameter estimation mean squared error of MoM and CLS vs. $T\times K$ for (\subref{fig:TK_error_binomial}) binomial noise and (\subref{fig:TK_error_additive}) additive Gaussian noise. }
      \label{fig:method_comparison}
    \end{figure*} 

\section{Theoretical Analysis}
\label{sec:theory}

We now analyze the method of moments estimator to provide theoretical guarantees on its performance. We prove consistency in two different settings: the \emph{time average} setting, when $T \rightarrow \infty$ for fixed $K$, and the \emph{ensemble average} setting, when $K \rightarrow \infty$ for fixed $T$. We also show that the mean squared estimation error for the moments decays as $\mathcal{O}(1/TK)$.

\paragraph{Consistency}
Fix any value of $N$ and let $\hat{P}_{T,K}$ be the estimator of Algorithm~\ref{alg:mom} when the process is observed for $T$ time steps and $K$ independent realizations. Assume the noise model $p(\y \mid \n)$ satisfies the conditions of Proposition~\ref{prop:noise} and that the noise model matrix $A$ and the population size are known. (If they are not known, but can be consistently estimated, which is usually the case, then the following results still hold.) We have the following result.
\begin{theorem}
The estimator of Algorithm~\ref{alg:mom} is consistent ($\hat{P}_{T,K}$ converges in probability to $P$) as one or both of $T$ and $K$ go to infinity.
\label{thm:consistency}
\end{theorem}
\begin{proof}
When $K \rightarrow \infty$, this is a simple consequence of the law of large numbers. We can view the sample moments in Algorithm~\ref{alg:mom} as averaging first over $k$ and then $t$. For example:
\[
\hat{\m}_\y = \frac{1}{T}\sum_{t=1}^T \hat{\m}_{\y,t}, \quad \hat{\m}_{\y,t} = \frac{1}{K}\sum_{k=1}^K\y_t^{(k)}.
\]
Since $\hat{\m}_{\y,t}$ is a sample average over the $K$ independent realizations, it converges to $\m_t = \E{\y_t}$. Furthermore, since the process is stationary, all terms $\hat{\m}_{\y,t}$ in the time average converge to the common value $\m = \E{\y_1}$. Thus $\hat{\m}_\y$ also converges to $\m$. The proof that the sample covariance $\frac{1}{(T-1)K} \sum_{t,k} \r_t^{(k)} (\r_t^{(k)})^T$ converges to $\cov(\y_1, \y_2)$ is similar. The estimate $\hat{P}_{T,K}$ is a deterministic function of these two sample moments, and so their convergence guarantees that $\hat{P}_{T,K}$ converges to $P$ as $K \rightarrow \infty$.

To prove consistency when $T \rightarrow \infty$, it is clearly enough to consider the case $K=1$, which we do now. We must argue that the time averages $\hat{\m}_\y = \frac{1}{T}\sum_{t=1}^T \y_t$ and $\hat{\Sigma}_\y = \frac{1}{T-1}\sum_{t=1}^{T}(\y_t - \hat{\m}_\y)(\y_t - \hat{\m}_\y)^T$
of the stationary process $\{\y_t\}$ converge to the true moments $\E{\y_1}$ and $\cov(\y_1, \y_2)$ as $T \rightarrow \infty$, which requires reasoning about ergodic properties of the process. A process is called \emph{mean-ergodic} if its time average converges to the population mean~\cite{Papoulis:1991aa}: applied to our problem, this is exactly equivalent to $\hat{\m}_\y$ being a consistent estimator of $\E{\y_1}$. Similarly, the covariance estimate $\hat{\Sigma}$ will converge to the population covariance if and only if the process $\{\Z_t\}$ is mean-ergodic, where $\Z_t = \y_t \y_{t+1}^T$. We will focus on proving that $\{\Z_t\}$ is mean-ergodic, which is a less obvious property than that of $\{\y_t\}$. It is enough to show that $\{z_t\}$ is mean-ergodic for an arbitrarily chosen entry of $z_t = y_t(i)y_{t+1}(j)$ of $\Z_t$. Furthermore, we will simplify the issue slightly by showing instead that this is true for $z_t = n_t(i)n_{t+1}(j)$. This is justified by our noise model: since noise is independent at each time step (condition (i) of Proposition~\ref{prop:noise}), one can show that the ergodic properties of the $\{\y_t\}$ process follow from those of the $\{\n_t\}$ process. 

Let $\gamma(k) = \cov\big(z_t, z_{t+k})$ be the autocovariance function of the process $\{z_t\}$. A sufficient condition for $\{z_t\}$ to be mean-ergodic is~\cite{Papoulis:1991aa}:
\begin{align}
\lim_{k \to \infty} \gamma(k) = 0 \label{mean_ergodic}
\end{align}
We give a detailed verification of this condition in the supplementary material. Intuitively, it is rather clear why this condition holds. Since each individual follows an ergodic Markov chain, the state $x_{t+k}$ becomes independent of $x_t$ as the time lag $k$ goes to infinity, and thus the autocovariance $\cov([x_t=i], [x_{t+k}=i'])$ goes to zero for any state-pair $(i,i')$. Equation~\eqref{mean_ergodic} asserts that an analogous autocovariance decay holds for products $n_t(i)n_{t+1}(j)$ of population counts at two adjacent time steps.
\end{proof}


\paragraph{Convergence Rates}
The consistency arguments in the previous section can also be modified to obtain convergence rates. 

\begin{theorem}
As the product $TK \rightarrow \infty$, the estimates $\hat{\m}_\y$ and $\hat{\Sigma}$  are unbiased and have variance (and thus mean squared error) $\mathcal{O}(1/TK)$.\label{thm:convergence}
\end{theorem}
\begin{proof}

We can modify our consistency argument above to compute the variance $\sigma^2$ of the time average estimate for $n_t(i)n_t(j)$ as $T \rightarrow \infty$ for a single chain ($K=1$), 
according to the following
formula~\cite{Papoulis:1991aa}:
\[
\sigma^2 = \frac{1}{T} \sum_{k = -(T-1)}^{T-1} \Big(1 - \frac{|k|}{T}\Big)\gamma(k)
\]
We show in the supplementary material that $|\gamma(k)| \leq C\alpha^k$ for some constants $\alpha \in (0,1)$ and $C > 0$, which implies that the sum above remains finite as $T \rightarrow \infty$ and thus $\sigma^2$ is $\mathcal{O}(1/T)$. Now, the final estimate is an average over $K$ iid estimators (one for each chain) with variance $\sigma^2$. Therefore, it has variance $\sigma^2/K = \mathcal{O}(1/TK)$.

\end{proof}
We have not presented a theoretical analysis of how the error rates of the moment estimates combine to give an error rate for the final parameter estimate $\hat{\P}$. Our experiments present evidence that the the mean squared error of $\hat{P}$ is $\mathcal{O}(1/TK)$.

\begin{figure*}[t]
  \centering
      \begin{subfigure}{.32\textwidth}
        \includegraphics[width=\textwidth]{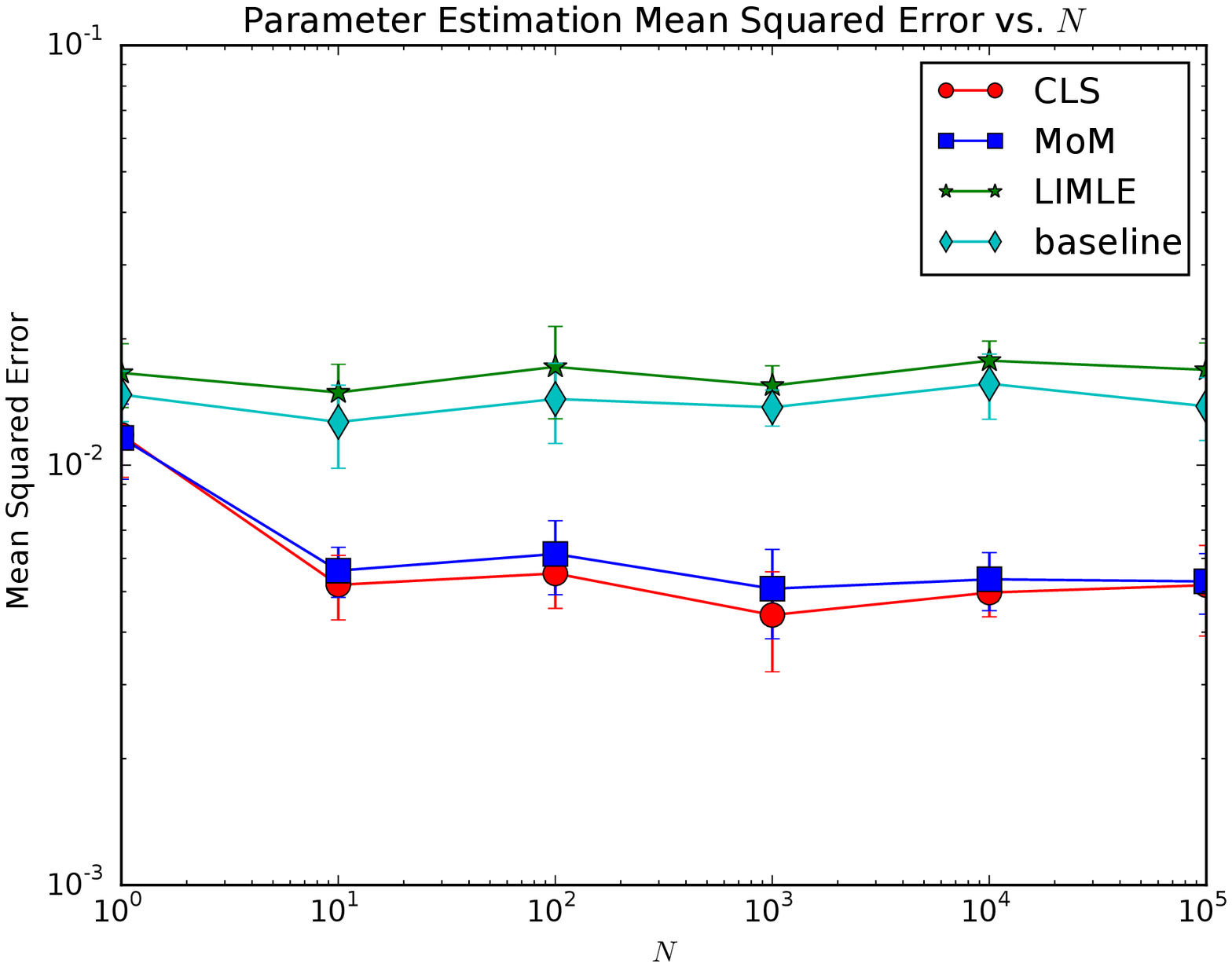}
      \caption{}
      \label{fig:error_N}  
      \end{subfigure}
      \hfill
      \begin{subfigure}{.32\textwidth}
        \includegraphics[width=\textwidth]{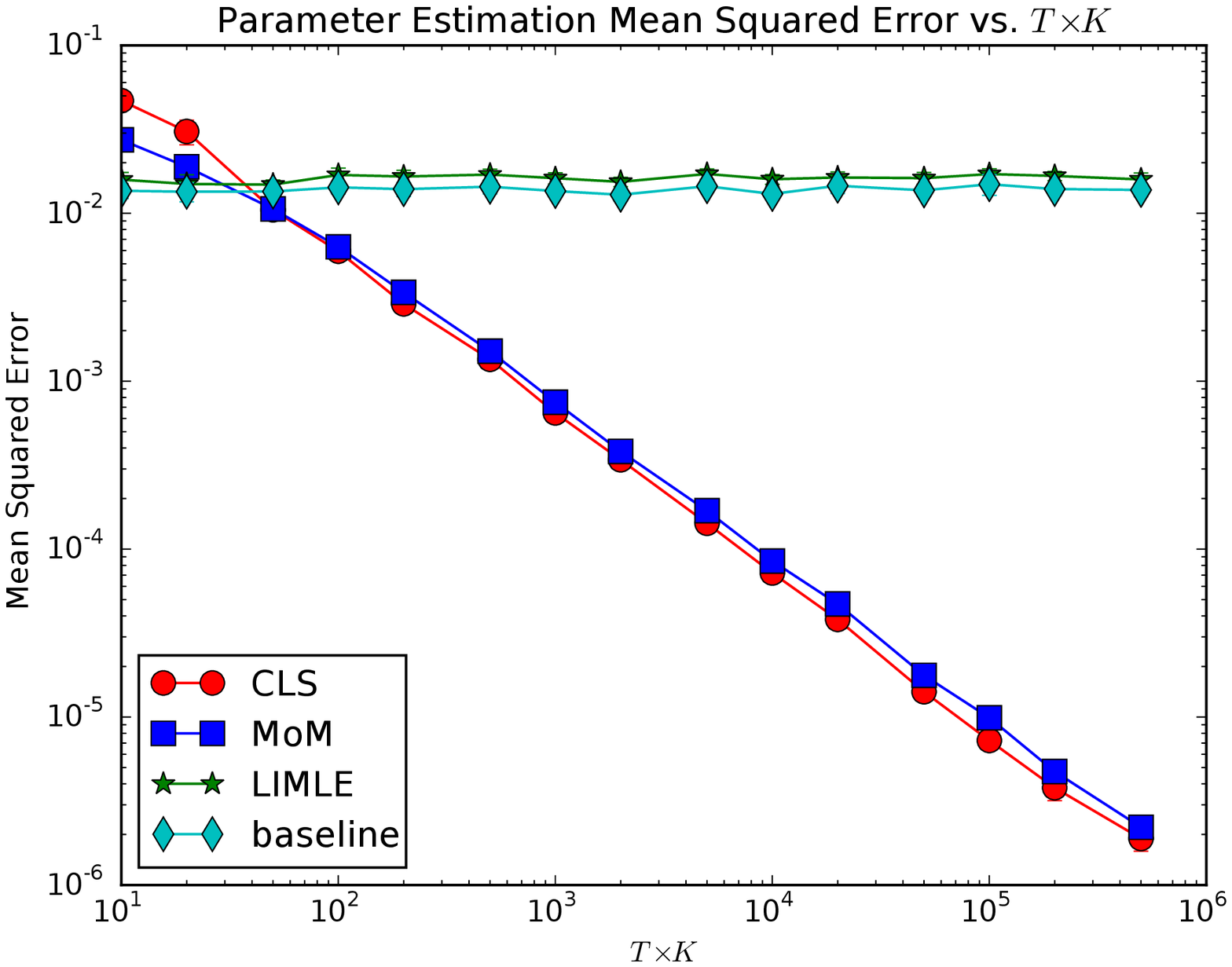}
        \caption{}
        \label{fig:baseline_parameter}
      \end{subfigure}
      \hfill
      \begin{subfigure}{.32\textwidth}
        \includegraphics[width=\textwidth]{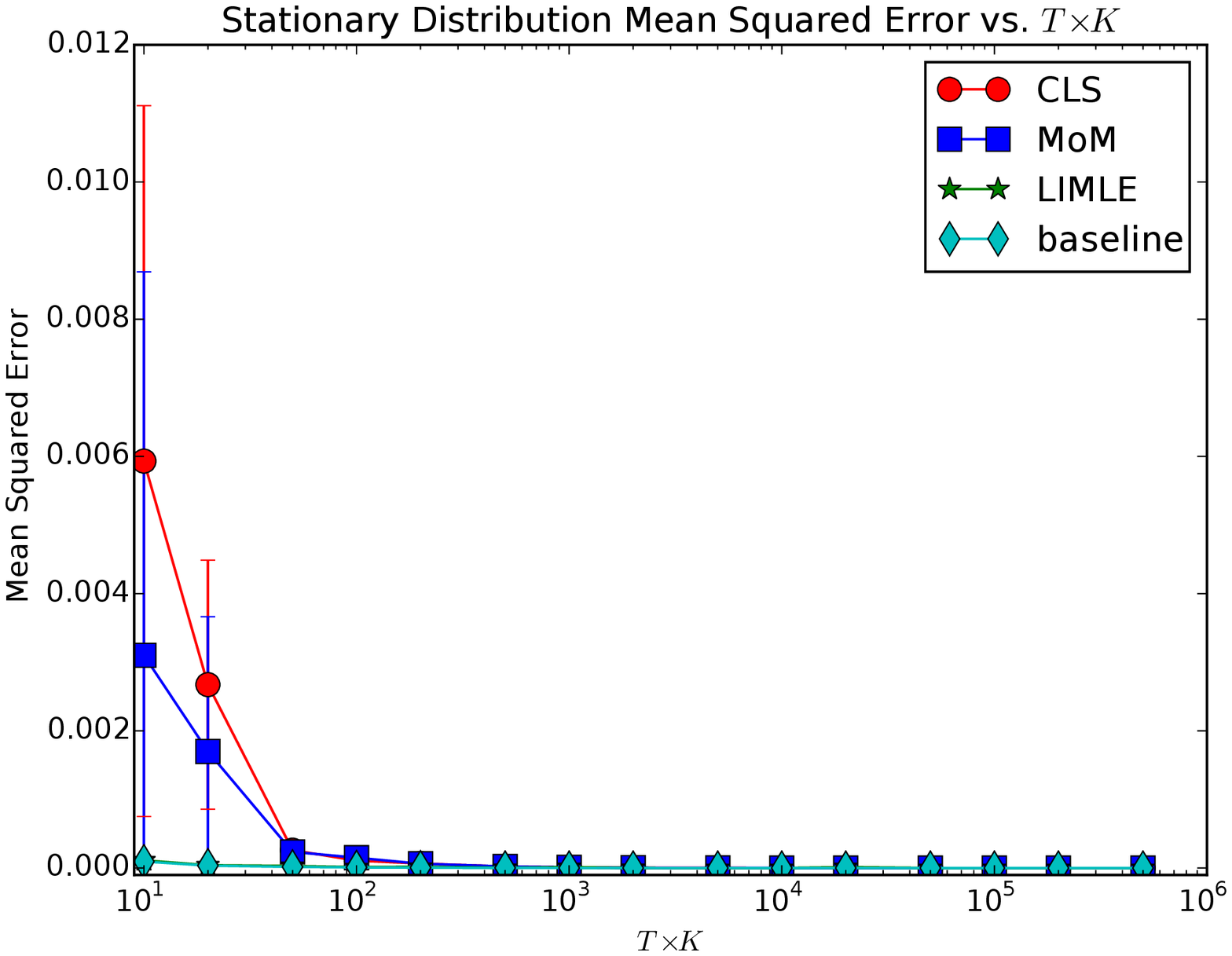}
        \caption{}
        \label{fig:baseline_sd}
      \end{subfigure}
      \caption{(\subref{fig:error_N}) Effect of population size $N$ on parameter estimation error (perfect observations, $K=1, T=100$). (\subref{fig:baseline_parameter}, \subref{fig:baseline_sd})
Mean squared error vs. $T\times K$ for
(\subref{fig:baseline_parameter}) parameter estimation and (\subref{fig:baseline_sd}) stationary distribution estimation. }
      \label{fig:baseline_plots}
\end{figure*}

\section{Experiments} 
\label{sec:experiments}

We conduct a number of experiments to validate our theoretical results and examine convergence rates of CLS, method of moments (hereafter, MoM), and additional baselines.

\textbf{Setup.}
We generate random $S\times S$ transition matrices $\P$ by drawing each row from the Dirichlet distribution using the mean and precision parameterization~\cite{minka2000} with mean vector $\mathbf{1}/S$ and precision $D$. In preliminary experiments, we varied $D$ and $S$ to examine the effectiveness of estimators for a broad range of qualitatively different transition models. For low $D$, the transitions from any given state are concentrated on a few states, so there is high dependence between time steps. For high $D$, each row of $P$ is very close to the mean, which implies that time steps are nearly independent. These experiments revealed that trends in estimation error of MoM and CLS are very consistent over a broad range of $S$ and $D$; therefore, we report results here only for $S=10$ and $D=0.5$.

Once the transition model $P$ is fixed, we compute the stationary distribution $\pib$ and generate the aggregate data by drawing $\n_1 \sim \text{Multinomial}(N, \pib)$ and then simulating the aggregate Markov process to generate $\n_2, \ldots, \n_T$. We then generate the noisy vectors $\y_1, \ldots, \y_T$ using the binomial and Gaussian noise models (parameters $\alpha$ and $\sigma^2$) described in Section~\ref{sub:noisy_observations_examples}. We repeat the entire process $K$ times to generate the final observed data $\{\y_t^{(k)}\}$.

\textbf{Baselines and Evaluation.}
Our main experiments compare MoM to CLS. We also compare to the \emph{limited information maximum likelihood estimator} (LIMLE) of \cite{MacRae:1977aa}, which was explicitly designed for noisy observations. LIMLE, however, assumes the marginals of process are time-varying and provide enough information to estimate the parameters: it finds $P$ to maximize the likelihood of the approximate model where $\n_t \sim \text{Multinomial}(N, \mub_t)$ independently for all $t$.
We evaluate the quality of each estimator $\hat{P}$ using entrywise mean squared error: $\frac{1}{S^2}\|\hat{P} - P\|_F^2$.

\textbf{Consistency and Convergence Rates.}
We first examined the asymptotic behavior of MoM and CLS as $T$ and $K$ grow by running both estimators for all combinations of $T \in \{10,100,1000,10000\}$ and $K \in \{1,2,5,10,20,50\}.$ We repeated each combination 10 times. Figures \ref{fig:TK_error_binomial} and \ref{fig:TK_error_additive} show estimation error for MoM and CLS plotted against the product $T \times K$ for binomial and Gaussian noise, respectively. Each data point averages over all trials for all combinations of $T$ and $K$ that yield the same product (error bars indicate the 95\% confidence interval of the mean error over that set of trials). For this experiment, we fixed $N=100$; we will later see that error has almost no dependence on $N$ (Figure~\ref{fig:error_N}).

In Figures \ref{fig:TK_error_binomial} and \ref{fig:TK_error_additive}, we see that the error of both MoM and CLS with perfect observations $(\alpha = 1.00, \sigma^2 = 0$) decays almost exactly as $1/TK$. We also see very clearly that CLS is not consistent with noisy observations; the estimation error flattens out at $TK$ grows, and at higher error levels when there is more noise. For binomial noise and Gaussian noise with $\sigma^2 = 1.0$, MoM retains its $1/TK$ convergence rate. For Gaussian noise with $\sigma^2 = 5.0$, the convergence of MoM appears to slow down for large $T \times K$; it may be the case, when the noise level is high enough, that the convergence rate of the final parameter estimate $\hat{P}$ is different from the $1/TK$ rate that we proved for the moment estimates.

\textbf{Effect of $N$.}
Figure~\ref{fig:error_N} shows the estimation error vs. the population size $N$ for each method. Perhaps surprisingly, $N$ has very little effect on estimation quality. Therefore, if we only observe aggregate data, there is no loss in estimation accuracy as the population size grows (a population of size $N=1$ is the same as one of size $N=10^6$).

\textbf{Individual vs. Aggregate Data.}
We can now address one of our questions from the outset: How does estimation with aggregate data compare to estimation with individual data? Suppose the total number of individuals is fixed and equal to $M$, and we consider the alternatives of aggregating the data into one population of size $M$ (i.e. $K=1, N=M)$ vs. observing individual data ($K=M, N=1$). Our results show that error decays as $O(1/TK)$ independently of $N$, so the error in the former case is approximately $M$ times the latter. In other words, estimation error increases by a factor of $M$ when we move from individual data to aggregate data in a population of size $M$; we would therefore need to observe the aggregate data of a system for a factor of $M$ additional time steps or $M$ additional independent realizations to achieve estimation quality comparable to observing individual data.


\textbf{Comparison to LIMLE}.
It is fairly clear that LIMLE, which approximates the model by one where each time step is independent, will fail to recover $\P$ when the process is stationary. This is because it cannot distinguish between two transition models that yield the same stationary distribution, and hence the same marginals at each time step.
Figures \ref{fig:baseline_parameter} confirms very clearly that LIMLE is not consistent. In fact, we can characterize LIMLE as searching for \emph{any} transition matrix $P$ that has the correct stationary distribution $\pib$. To verify this, we do two things. First, we introduce a naive baseline that does exactly the same thing: it estimates the stationary distribution $\hat{\pib}$ by averaging over $t$ and $k$, and then sets $\hat{P}$ to the matrix with each row equal to $\hat{\pib}$. (The result is a model where $x_t$ is an independent draw from $\hat{\pib}$ at each time step). Figure~\ref{fig:baseline_parameter} shows that this baseline actually outperforms LIMLE slightly in our stationary setting. Second, we compare all methods in terms of their error estimating the correct \emph{stationary distribution}. Figure~\ref{fig:baseline_sd} shows that LIMLE indeed quickly converges to the correct stationary distribution, even though it fails to estimate $P$. CLS and MoM are slightly slower to converge to the correct stationary distribution, but consistently estimate $P$.

\bibliography{mom_aistats2016_camera_ready}{}
\bibliographystyle{plainnat}

\newpage 
\appendix

\chead{\textbf{Consistently Estimating Markov Chains with Noisy Aggregate Data} \\ Garrett Bernstein and Daniel Sheldon \\University of Massachusetts Amherst}
\setlength{\headheight}{35pt}

\onecolumn

\section{Extra Proofs}

\subsection{Additional Details for Proof of Proposition~\ref{prop:P}} 
\label{sub:additional_details_of_proof_of_proposition}


We wish to show that $\Lambda_t = N\big(\diag(\mub_t) + (N-1)\mub_t \mub_t^T\big)$ is invertible. The Sherman-Morrison formula \cite{Sherman:1950aa} gives the inverse of a matrix that is equal to an invertible matrix ($\diag(\mub_t)$) plus a rank-one matrix (the rank-one outer product of $\mub_t$).

Specifically, let $D = \diag(\mub_t)$ and then we have

\begin{align*}
\Lambda_t^{-1} 
&= N^{-1}\big(D + (N\!-\!1)\mub_t \mub_t^T\big)^{-1} \\
&= N^{-1}\Big(D^{-1} - 
  \frac{
    (N-1)D^{-1}\mub_t \mub_t^T D^{-1}
  }
  {
    1 + (N-1) \mub^T D^{-1} \mub_t
  }\Big) \\
&= N^{-1} \Big( D^{-1} - \frac{N-1}{N} \mathbf{1} \mathbf{1}^T \Big).
\end{align*}
We have used the fact that $D^{-1} \mub_t = \mathbf{1}$ (the all ones vector) and that $\mathbf{1}^T \mub_t = 1$, since $\mub_t$ is a vector of marginal probabilities.



\subsection{Proof of Proposition~\ref{prop:noise}}
\begin{proof}
Using condition (ii) of the proposition, we write:
\[
\E{\y_t} = \mathbb{E}\big[\E{\y_t | \n_t}\big] = \E{\A_t \n_t} = \A_t \E{\n_t}.
\]		
Since $A_t$ is invertible, we have $\E{\n_t} = \A_t^{-1}\E{\y_t}$, which proves conclusion (i).

For the non-central second moments, we have for $s \neq t$:
\begin{align}
\label{eq:conditional_var} \E{\y_s \y_t^\T}
&= \mathbb{E}\big[\E{ \y_s \y_t^\T | \n_s, \n_t}\big] \\
\notag &= \E{\E{ \y_s | \n_s} \cdot \E{\y_t^\T | \n_t}} \\
\notag &= \E{ (\A_s \n_s) (\A_t \n_t)^\T} \\
\notag &= \A_s \E{\n_s \n_t^\T} \A_t^\T
\end{align}
The second line uses condition (i) of the proposition: $\y_s$ and $\y_t$ are conditionally independent given $\n_s$ and $\n_t$ if $s \neq t$. Therefore, we have $\E{\n_s \n_t^\T} = \A_s^{-1}\E{\y_s \y_t^\T}\A_t^{-\T}$, which proves conclusion (iii).

For conclusion (ii), we have for $s \neq t$:
\begin{align*}
\cov(\y_s,\y_t)
&= \E{\y_s\y_t^\T} - \E{\y_s}\E{\y_t}^\T \\
&= \A_s \E{\n_s \n_t^\T} \A_t^\T - \A_s \E{\n_s}\E{\n_t}^\T A_t^\T \\
&= \A_s \big(\E{\n_s \n_t^\T} - \E{\n_s}\E{\n_t}^\T\big) A_t^\T \\
&= \A_s \cov(\n_s,\n_t) \A_t^\T,
\end{align*}
so that $\cov(\n_s, \n_t) = A_s^{-1}\cov(\y_s, \y_t) A_t^{-T}$, which completes the proof.
\end{proof}

\subsection{Additional Details for Proof of Theorem~\ref{thm:consistency}}
\label{sub:consistency_appendix}
We wish to show that $\lim_{k \rightarrow \infty} \gamma(k) = 0$, where $\gamma(k) = \cov\big( n_t(i)n_{t+1}(j),n_{t+k}(i)n_{t+k+1}(j) \big)$. We have

\begin{align*}
\gamma(k) 
&= \cov(n_t(i)n_{t+1}(j),n_{t+k}(i)n_{t+k+1}(j)) \\
&= \cov\Bigg(\sum_{a=1}^N \sum_{b=1}^N[x_t^{(a)}=i][x_{t+1}^{(b)}=j],\,
\sum_{c=1}^N \sum_{d=1}^N [x_{t+k}^{(c)}=i][x_{t+k+1}^{(d)}=j]\Bigg) \\
&= \sum_{a,b,c,d=1}^N \cov\Big( [x_t^{(a)}=i][x_{t+1}^{(b)}=j], \,\,[x_{t+k}^{(c)}=i][x_{t+k+1}^{(d)}=j]\Big)
\end{align*}
It is enough to show that the covariance in the summand goes to zero for any choice of four individuals $a,b,c,d \in \{1,\ldots,N\}$. Clearly, it is equal to zero when the individuals $\{a,b\}$ do not overlap with $\{c,d\}$, because individuals are independent. We will verify that the covariance goes to zero for the choice $a = b = c = d := m$, which, since it involves only a single individual, is the case with the \emph{greatest} dependence between times $t$ and $t+k$. Verifying the statement for other combinations of $a,b,c,d$ is similar. Because we are considering a single individual $m$, we now drop the superscript and write $x_t := x_t^{(m)}$. We can rewrite the covariance as:
\begin{align}
\notag \cov\Big( &[x_t=i][x_{t+1}=j], \,\,[x_{t+k}=i][x_{t+k+1}=j]\Big) \\
\notag &=  \mathbb{E}\Big[[x_t=i][x_{t+1}=j][x_{t+k}=i][x_{t+k+1}=j]\Big]
 - \mathbb{E}\Big[[x_t=i][x_{t+1}=j]\Big] \mathbb{E}\Big[[x_{t+k}=i][x_{t+k+1}=j]\Big]  \\
\notag &= \Pr{x_t=i,x_{t+1}=j,x_{t+k}=i,x_{t+k+1}=j} 
- \Pr{x_t=i,x_{t+1}=j}\Pr{x_{t+k}=i,x_{t+k+1}=j} \\
\label{eq:cov} 
&= \mu(i,j) \cdot (P^{k-1})_{ji} \cdot P(i,j) \;\; - \;\; \mu(i,j)^2\,.
\end{align}
In the last line, we apply several facts about the Markov chain. Here, $\mu(i,j) = \Pr{x_t = i, x_{t+1}=j}$ is the (time-independent) pairwise marginal, $(P^{k-1})_{ji} = \Pr{ x_{t+k} = i \mid x_{t+1} = j}$ and $P(i,j) = \Pr{ x_{t+k+1} = j \mid x_{t+k} = i}$. Since the Markov chain is ergodic, $\lim_{k \rightarrow \infty} (P^{k-1})_{ji} = \pi(i)$, so the first term of Equation~\eqref{eq:cov} becomes:
\[
\lim_{k \rightarrow \infty} \mu(i,j)(P^{k-1})_{ji} P(i,j) = \mu(i,j)\pi(i) P(i,j) = \mu(i,j)^2. \label{eq:stationary_limit}
\]
Putting it all together, we see that the limit as $k$ goes to infinity of the covariance in Equation~\eqref{eq:cov} is $\mu(i,j)^2 - \mu(i,j)^2 = 0$, as desired. This completes the proof.

\subsection{Additional Details for Proof of Theorem~\ref{thm:convergence}} 
\label{sub:additional_details_for_proof_of_theorem_thm:convergence}

We wish to show that $|\gamma(k)|$ decays exponentially to zero as $k \rightarrow \infty$, where $|\gamma(k)| = |\cov\big( n_t(i)n_{t+1}(j),n_{t+k}(i)n_{t+k+1}(j) \big)|$. We follow the exact same steps in Section \ref{sub:consistency_appendix} up through Equation \eqref{eq:cov} where we instead desire $|(P^{k-1})_{ji} - \pi(i)| \leq C\alpha^k$ for some constants $\alpha \in (0,1)$ and $C > 0$. This is proved for irreducible and aperiodic $P$ as Theorem 4.9 in \cite{Levin:2009aa}. Using this fact together with Equation~\eqref{eq:cov}, we have:
\begin{align*}
\left|\cov\Big([x_t=i][x_{t+1}=j], \,\,[x_{t+k}=i][x_{t+k+1}=j]\Big)\right| 
   &= \left| \mu(i,j) \cdot (P^{k-1})_{ji} \cdot P(i,j) \;\; - \;\; \mu(i,j)^2\right| \\
   &= \left| \mu(i,j) \cdot (P^{k-1})_{ji} \cdot P(i,j) \;\; - \;\; \mu(i,j) \cdot \pi(i) \cdot P(i,j) \right| \\
   &= \mu(i,j) \cdot \left| (P^{k-1})_{ji} - \pi(i) \right| \cdot P(i,j) \\
   &\leq \mu(i,j) \cdot C \alpha^k \cdot P(i,j) \\
   &= C' \alpha^k.
\end{align*}
for $C' = \mu(i,j) P(i,j) C$. Thus the result is proved.

\end{document}